\documentclass[runningheads]{llncs}
\usepackage[width=122mm,left=12mm,paperwidth=146mm,height=193mm,top=12mm,paperheight=217mm]{geometry}

\usepackage[tocindentauto]{tocstyle}
\usepackage{etex}
\usepackage{url}
\usepackage[numbers]{natbib}
\renewcommand\bibname{References}
\usepackage{cvpr-abbriv}
\usepackage{epsfig}
\usepackage{graphicx}
\usepackage{amsmath}
\usepackage{amssymb}
\usepackage{array}
\usepackage{tabularx}
\usepackage{url}
\usepackage{graphicx}
\usepackage{amsmath,amssymb} 
\makeatletter
\let\proof\@undefined                        
\let\endproof\@undefined                  
\makeatother
\usepackage{amsthm}
\usepackage{amsthm}
\usepackage{amsxtra}
\usepackage{stmaryrd}
\usepackage{comment}
\usepackage[ruled,vlined,linesnumbered,procnumbered,longend]{algorithm2e}

\SetCommentSty{mycommfont}
\usepackage{multirow,rotating}
\usepackage{array}
\usepackage{afterpage}
\usepackage{dblfloatfix}
\usepackage{setspace}
\usepackage{floatrow}
\usepackage{booktabs}
\usepackage{color, colortbl}
\usepackage{tabularx,calc}
\usepackage{tikz}                 
\usetikzlibrary{patterns,fit,external}
\usepackage{pgfplots}             
\usepackage{overpic}
\usepackage{tocloft}
\usepackage{bbm}
\usepackage[absolute]{textpos}
\usepackage{atbeginend}
\usepackage{units}
\usepackage{xifthen}

\usepackage{times}
\usepackage{tikzscale}
\usepackage{marginnote}
\usepackage[textwidth=2cm,figwidth=\columnwidth,colorinlistoftodos]{todonotes}
\makeatletter
\renewcommand{\todo}[2][]{\tikzexternaldisable\@todo[#1]{#2}\tikzexternalenable}
\makeatother

\newcounter{mycomment} 

\usepackage{booktabs}
\usepackage{multirow}
\usepackage{tabularx}
\usepackage{array}
\usepackage{floatrow}
\usepackage[
  position=bottom,
  font=small,
  labelfont={bf},
]{caption,subfig}
\usepackage{longtable}
\usepackage{tabu} 
\usepackage[most]{tcolorbox}
\usepackage{pbox}
\newlength{\luw}
\newlength{\luh}

\usepackage{xifthen}

\usepackage[draft=false, pagebackref=false,breaklinks=true,colorlinks,bookmarks=false,pdftex]{hyperref}
\hypersetup{hidelinks,colorlinks=false,linkbordercolor={1 1 1}}
\DeclareMathAlphabet{\mathcalligra}{T1}{calligra}{m}{n}
\DeclareMathAlphabet{\mathantt}{OT1}{antt}{li}{it}
\DeclareMathAlphabet{\mathpzc}{OT1}{pzc}{m}{it}

\let\emptyset\varnothing

\renewcommand{\mid}{\:|\,}

\newcommand{\Natural}{\mathbb{N}}

\newcommand{\A}{\mathcal{A}}
\newcommand{\Bool}{\mathbb{B}}

\def\G{\mathcal{G}}

\def\C{\mathcal{C}}

\newcommand{\V}{\mathcal{V}}
\newcommand{\I}{\mathcal{I}}

\renewcommand{\P}{\mathcal{P}}
\newcommand{\E}{\mathcal{E}}

\renewcommand{\L}{\mathcal{L}}

\newcounter{myRomanCounter}

\newcommand{\x}{{x}}

\def\c{\textsc{c}}

\newcommand{\gray}{\color[rgb]{0.5,0.5,0.5}}
\newcommand{\red}{\color[rgb]{1,0,0}}

\renewcommand*{\paragraph}[1]{\par\noindent{\normalsize\bf #1}\,\xspace}

\def\ij{i,j}

\def\phi{\delta}
\def\epsilon{\varepsilon}

\newcommand{\AND}{\mbox{ \rm and }}

\newcommand{\Section}[1]{\cref{#1}}
\makeatletter

\def\S{\mathcal{S}}
\makeatother


\newcommand{\revisit}[1][]{%
\ifthenelse{\equal{#1}{}}{
\ensuremath{\red \triangle}\xspace}{%
{\ensuremath{\red \rhd}\xspace}%
{\gray #1}%
{\ensuremath{\red \lhd}\xspace}%
}%
}

\def\anchor [#1]#2{%
\phantomsection{}#1\label{#2}%
\def\arga{#2}%
\global\expandafter\def\csname#2\endcsname{%
\hyperref[#2]{#1}\xspace%
}%
}%

\def\codefunction [#1]#2{%
\phantomsection{}\label{#2}{\ttfamily #1\xspace}%
\def\arga{#2}%
\global\expandafter\def\csname#2\endcsname{%
\hyperref[#2]{\ttfamily #1}\xspace%
}%
}

\usepackage{array}
\newcolumntype{L}[1]{>{\raggedright\let\newline\\\arraybackslash\hspace{0pt}}m{#1}}
\newcolumntype{C}[1]{>{\centering\let\newline\\\arraybackslash\hspace{0pt}}m{#1}}
\newcolumntype{R}[1]{>{\raggedleft\let\newline\\\arraybackslash\hspace{0pt}}m{#1}}

\SetKwFor{forever}{while true}{}{end while}

\def\epsilon{\varepsilon}

\newcommand{\leqAP}{\mathbin{{\leq}_{\text{\tiny AP}}}}

\newlength{\myskip}
\renewenvironment{enumerate}%
  {\begin{list}{\arabic{enumi}.}%
     {\topsep=0in\itemsep=0in\parsep=0pt\partopsep=0in\usecounter{enumi}}%
   }{\end{list}}

\renewenvironment{itemize}%
  {\begin{list}{$\bullet$}%
     {\topsep=0in\itemsep=0pt\parsep=0pt\partopsep=0in\usecounter{itemi}}%
   }{\end{list}\addvspace{0pt}}

\SetAlgoSkip{skipalgomyskip}
\SetAlgoInsideSkip{}

\makeatletter
\let\corollary\@undefined
\let\endcorollary\@undefined
\let\definition\@undefined
\let\enddefinition\@undefined
\let\proof\@undefined
\let\endproof\@undefined
\let\theorem\@undefined
\let\c@theorem\@undefined
\let\endtheorem\@undefined
\let\lemma\@undefined
\let\endlemma\@undefined
\let\example\@undefined
\let\c@example\@undefined
\let\endexample\@undefined
\let\remark\@undefined
\let\endremark\@undefined
\let\proposition\@undefined
\let\endproposition\@undefined
\let\property\@undefined
\let\endproperty\@undefined
\makeatother

\usepackage{thmtools}

\newtheoremstyle{tightItalic}
  {0.5\myskip}
  {0\myskip}
  {}
  {}
  {\itshape}
  {.}
  { }
  {}

\newtheoremstyle{tightBf}
  {0.5\myskip}
  {0.5\myskip}
  {}
  {}
  {\bf}
  {.}
  {.5em}
  {}

\theoremstyle{definition}
\theoremstyle{tightBf}
\declaretheorem[style=tightBf,parent=section]{thm}
\numberwithin{thm}{section}
\declaretheorem[style=tightBf,sibling=thm,name=Theorem]{theorem}
\declaretheorem[style=tightBf,sibling=theorem,name=Lemma]{lemma}
\declaretheorem[style=tightBf,sibling=theorem,name=Definition]{definition}
\declaretheorem[style=tightBf,sibling=theorem,name=Corollary]{corollary}
\declaretheorem[style=tightBf,sibling=theorem,name=Proposition]{proposition}

\declaretheorem[style=tightBf,parent=section,name=Example]{example}
\numberwithin{example}{section}

\declaretheorem[style=tightItalic,sibling=theorem,name=Remark]{remark}

\theoremstyle{tightItalic}
\newtheorem*{proof}{Proof}
\BeforeEnd{proof}{\qed}
\newtheorem*{proofsketch}{Proof Sketch}

\usepackage{thmtools, thm-restate}
\usepackage[capitalise,nameinlink]{cleveref}
\crefformat{equation}{(#2#1#3)}
\crefname{section}{Section}{Sections}
\Crefname{section}{Section}{Sections}

\usepackage[bottom]{footmisc}

\begin{document}

\begin{textblock}{14}(1,0.4)
\noindent
Author's version. The final publication is available at \url{link.springer.com}.
\end{textblock}


%
%

\pagestyle{headings}
\mainmatter
\def\ECCV16SubNumber{1}  

\title{Complexity of Discrete Energy Minimization Problems}



\author{Mengtian Li\textsuperscript{ 1} \qquad Alexander Shekhovtsov\textsuperscript{ 2} \qquad Daniel Huber\textsuperscript{ 1}}
\institute{\textsuperscript{ 1}The Robotics Institute, Carnegie Mellon University\\\textsuperscript{ 2}Institute for Computer Graphics and Vision, Graz University of Technology \\
\email{ \textsuperscript{ 1}\{mtli, dhuber\}@cs.cmu.edu \quad \textsuperscript{ 2}shekhovtsov@icg.tugraz.at}}

\authorrunning{Mengtian Li, Alexander Shekhovtsov and Daniel Huber}



\maketitle

\newboolean{InAppendix}
\setboolean{InAppendix}{false}
\begin{abstract}


Discrete energy minimization is widely-used in computer vision and machine learning for problems such as MAP inference in graphical models.  The problem, in general, is notoriously intractable, and finding the global optimal solution is known to be NP-hard. However, is it possible to approximate this problem with a reasonable ratio bound on the solution quality in polynomial time?  We show in this paper that the answer is no.  Specifically, we show that general energy minimization, even in the 2-label pairwise case, and planar energy minimization with three or more labels are exp-APX-complete.  This finding rules out the existence of any approximation algorithm with a sub-exponential approximation ratio in the input size for these two problems, including constant factor approximations. Moreover, we collect and review the computational complexity of several subclass problems and arrange them on a complexity scale consisting of three major complexity classes -- PO, APX, and exp-APX, corresponding to problems that are solvable, approximable, and inapproximable in polynomial time. Problems in the first two complexity classes can serve as alternative tractable formulations to the inapproximable ones. This paper can help vision researchers to select an appropriate model for an application or guide them in designing new algorithms.

\keywords{Energy minimization, complexity, NP-hard, APX, exp-APX, NPO, WCSP, min-sum, MAP MRF, QPBO, planar graph}

\end{abstract}

\setcounter{tocdepth}{2}
\makeatletter
\renewcommand*\l@author[2]{}
\renewcommand*\l@title[2]{}
\makeatletter

\section{Introduction}

Discrete energy minimization, also known as min-sum labeling~\cite{Werner-PAMI07} or weighted constraint satisfaction (WCSP)\footnote{WCSP is a more general problem, considering a bounded plus operation.
It is itself a special case of valued CSP, where the objective takes values in a more general valuation set.}~\cite{jeavons2014complexity}, is a popular model for many problems in computer vision, machine learning, bioinformatics, and natural language processing. In particular, the problem arises in maximum a posteriori (MAP) inference for Markov (conditional) random fields (MRFs/CRFs)~\cite{Lauritzen96}.  In the most frequently used pairwise case, the {\em discrete energy minimization problem} (simply ``energy minimization'' hereafter) is defined as
\begin{align} \label{eq:1}
\min_{x\in\mathcal{L}^\V} \sum_{u\in\mathcal{V}} f_u(x_u) + \sum_{(u,v)\in\mathcal{E}} f_{uv}(x_u,x_v),
\end{align}
where $x_u$ is the label for node $u$ in a graph $\mathcal{G}=(\mathcal{V}, \mathcal{E})$. When the variables $x_u$ are binary (Boolean): $\L = \Bool = \{0,1\}$, the problem can be written as a quadratic polynomial in $x$~\cite{BorosHammer01} and is known as quadratic pseudo-Boolean optimization (QPBO)~\cite{BorosHammer01}. 

In computer vision practice, energy minimization has found its place in semantic segmentation~\cite{ren2012rgb}, pose estimation \cite{yang2011articulated}, scene understanding  \cite{schwing2012efficient}, depth estimation \cite{liu2010single}, optical flow estimation \cite{xu2012motion}, image in-painting \cite{shekhovtsov-2012-curvature}, and image denoising \cite{barbu2009learning}.
For example, tree-structured models have been used to estimate pictorial structures such as body skeletons or facial landmarks~\cite{yang2011articulated}, multi-label Potts models have been used to enforce a smoothing prior for semantic segmentation~\cite{ren2012rgb}, and general pairwise models have been used for optimal flow estimation~\cite{xu2012motion}.
However, it may not be appreciated that the energy minimization formulations used to model these vision problems have greatly varied degrees of tractability or {\em computational complexity}. For the three examples above, the first allows efficient exact inference, the second admits a constant factor approximation, and the third has no quality guarantee on the approximation of the optimum. 

%

The study of complexity of energy minimization is a broad field. Energy minimization problems are often intractable in practice except for special cases.  While many researchers analyze the time complexity of their algorithms (e.g., using big O notation), it is beneficial to delve deeper to address the difficulty of the underlying problem.  The two most commonly known complexity classes are P (polynomial time) and NP (nondeterministic polynomial time: all decision problems whose solutions can be verified in polynomial time). However, these two complexity classes are only defined for {\em decision} problems. The analogous complexity classes for {\em optimization} problems are PO (P optimization) and NPO (NP optimization: all optimization problems whose solution feasibility can be verified in polynomial time). Optimization problems form a superset of decision problems, since any decision problem can be cast as an optimization over the set $\{$yes, no$\}$, \ie, P $\subset$ PO and NP $\subset$ NPO. The NP-hardness of an optimization problem means it is at least as hard as (under Turing reduction) the hardest decision problem in the class NP. If a problem is NP-hard, then it is not in PO assuming P $\neq$ NP. 

Although optimal solutions for problems in NPO, but not in PO, are intractable, it is sometimes possible to guarantee that a good solution (i.e., one that is worse than the optimal by no more than a given factor) can be found in polynomial time.  These problems can therefore be further classified into class APX (constant factor approximation) and class exp-APX (inapproximable) with increasing complexity (Figure~\ref{fig:hardnessaxis}).  We can arrange energy minimization problems on this more detailed complexity scale, originally established in \cite{ausiello1999complexity}, to provide vision researchers a new viewpoint for complexity classification, with a focus on NP-hard optimization problems.

In this paper, we make three core contributions, as explained in the next three paragraphs. First, we prove the inapproximability result of QPBO and general energy minimization. Second, we show that the same inapproximability result holds when restricting to planar graphs with three or more labels. In the proof, we propose a novel micro-graph structure-based reduction that can be used for algorithmic design as well. Finally, we present a unified framework and an overview of vision-related special cases where the energy minimization problem can be solved in polynomial time or approximated with a constant, logarithmic, or polynomial factor.

\begin{figure}
\begin{center}
   \includegraphics[width=0.6\linewidth]{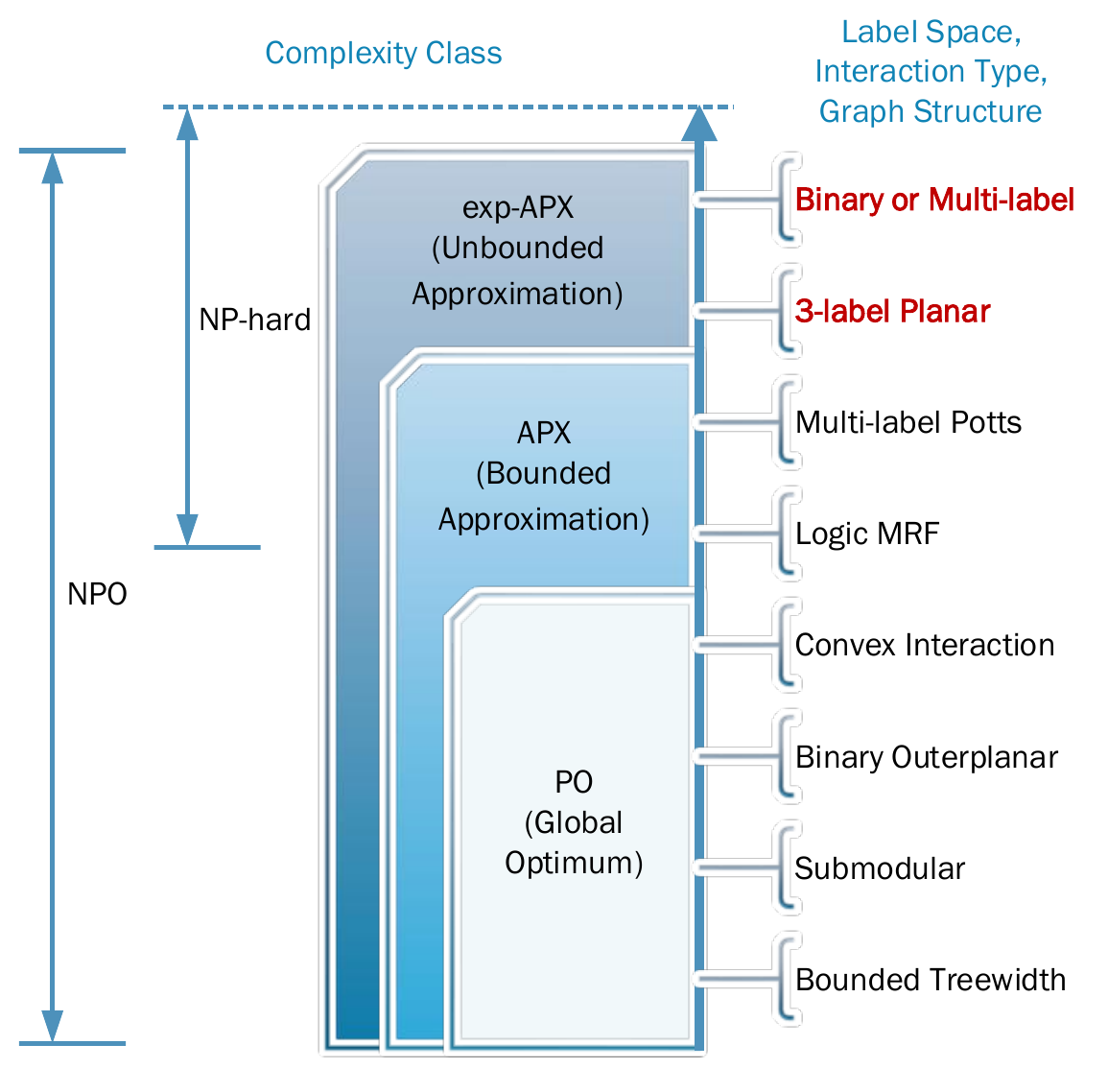}
\end{center}
    \caption{Discrete energy minimization problems aligned on a complexity axis. Red/boldface indicates new results proven in this paper. This axis defines a partial ordering, since problems within a complexity class are not ranked. Some problems discussed in this paper are omitted for  simplicity.}
\label{fig:hardnessaxis}
\end{figure}

\textbf{Binary and multi-label case}  (Section~\ref{sec:gencase}). It is known that QPBO (2-label case) and the general energy minimization problem (multi-label case) are NP-hard~\cite{BorosHammer02}, because they generalize such classical NP-hard optimization problems on graphs as vertex packing (maximum independent set) and the minimum and maximum cut problems~\cite{Karp-72}.
In this paper, we show a stronger conclusion. \emph{We prove that QPBO as well as general energy minimization are complete (being the hardest problems) in the class exp-APX.} Assuming P $\neq$ NP, this implies that a polynomial time method
cannot have a guarantee of finding an approximation within a constant factor of the optimal, and in fact, the only possible factor in polynomial time is exponential in the input size. 
In practice, this means that a solution may be essentially arbitrarily bad.

\textbf{Planar three or more label case}  (Section~\ref{sec:plcase}). Planar graphs form the underlying graph structure for many computer vision and image processing tasks. It is known that efficient exact algorithms exist for some special cases of planar 2-label energy minimization problems~\cite{Schraudolph-10}.  In this paper, we show that for the case of three or more labels, planar energy minimization is exp-APX-complete, which means these problems are as hard as general energy minimization. It is unknown that whether a constant ratio approximation exists for planar 2-label problems in general.

\textbf{Subclass problems}  (Section~\ref{sec:specialcases}). Special cases for some energy minimization algorithms relevant to computer vision are known to be tractable. However, detailed complexity analysis of these algorithms is patchy and spread across numerous papers.  In Section~\ref{sec:specialcases}, we classify the complexity of these subclass problems and illustrate some of their connections.  Such an analysis can help computer vision researchers become acquainted with existing complexity results relevant to energy minimization and can aid in selecting an appropriate model for an application or in designing new algorithms.

\subsection{Related Work}\label{sec:related}

%

Much of the work on complexity in computer vision has focused on experimental or empirical comparison of inference methods, including influential studies on choosing the best optimization techniques for specific classes of energy minimization problems~\cite{Szeliski08-PAMI,kappes2015comparative} and the PASCAL Probabilistic Inference Challenge, which focused on the more general context of inference in graphical models~\cite{PIC}. In contrast, our work focuses on theoretical computational complexity, rather than experimental analysis.

On the theoretical side, the NP-hardness of certain energy minimization problems is well studied. It has been shown that 2-label energy minimization is, in general, NP-hard, but it can be in PO if it is submodular~\cite{Kolmogorov02:regular-pami} or outerplanar~\cite{Schraudolph-10}. For multi-label problems, the NP-hardness was proven by reduction from the NP-hard multi-way cut problem~\cite{boykov2001approximate}. These results, however, say nothing about the complexity of {\em approximating} the global optimum for the intractable cases. The complexity involving approximation has been studied for classical combinatorial problems, such as MAX-CUT and MAX-2SAT, which are known to be APX-complete~\cite{Papadimitriou-91}. QPBO generalizes such problems and is therefore APX-hard. This leaves a possibility that QPBO may be in APX, \ie, approximable within a constant factor.


Energy minimization is often used to solve MAP inference for undirected graphical models. In contrast to scarce results for energy minimization and undirected graphical models, researchers have more extensively studied the computational complexity of approximating the MAP solution for {\em Bayesian networks}, also known as {\em directed graphical models}~\cite{kwisthout2015tree}. Abdelbar and Hedetniemi first proved the NP-hardness for approximating the MAP assignment of directed graphical models in the value of probability, \ie, finding $x$ such that
\begin{align}\label{p-approx-ratio}
\frac{p(x^*)}{p(x)} \leq r(n)
\end{align}
with a constant or polynomial ratio $r(n) \geq 1$ is NP-hard and showing that this problem is poly-APX-hard~\cite{Abdelbar-98}. 
The probability approximation ratio is closest to the energy ratio used in our work, but other approximation measures have also been studied.  \citeauthor{Kwisthout-11} showed the NP-hardness for approximating MAPs with the measure of additive value\mbox{-,} structure-, and rank-approximation~\cite{Kwisthout-11,Kwisthout-13,kwisthout2015tree}.
He also investigated the hardness of expectation-approximation of MAP and found that no randomized algorithm can expectation-approximate MAP in polynomial time with a bounded margin of error unless NP $\subseteq$ BPP, an assumption that is highly unlikely to be true~\cite{kwisthout2015tree}.


Unfortunately, the complexity results for directed models do not readily transfer to undirected models and vice versa. In directed and undirected models, the graphs represent different conditional independence relations, thus the underlying family of probability distributions encoded by these two models is distinct, as detailed in 
\ifdefined\NOAPPENDIX
    Appendix B.
\else
    \Section{sec:BayesNet}.
\fi
However, one can ask similar questions on the hardness of undirected models in terms of various approximation measures. In this work, we answer two questions, ``How hard is it to approximate the MAP inference in the ratio of energy (log probability) and the ratio of probability?'' The complexity of structure-, rank-, and expectation-approximation remain open questions for energy minimization.

\section{Definitions and Notation} \label{sec:prelim}

There are at least two different sets of definitions of what is considered an NP optimization problem~\cite{orponen1987approximation, ausiello1999complexity}. Here, we follow the notation of Ausiello et al~\cite{ausiello1999complexity} and restate the definitions needed for us to state and prove our theorems in Sections~\ref{sec:gencase} and \ref{sec:plcase} with our explanation of their relevance to our proofs.

\begin{definition}[Optimization Problem, \cite{ausiello1999complexity} Def. 1.16]An {\em optimization problem} $\P$ is characterized by a quadruple $(\I,\S,m,\rm{goal})$ where
\begin{enumerate}
\item $\I$ is the set of instances of $\P$.
\item $\S$ is a function that associates to any input instance $x\in\I$ the set of {\em feasible solutions} of $x$.
\item $m$ is the {\em measure} function, defined for pairs $(x,y)$ such that $x\in\I$ and $y\in\S(x)$. For every such pair $(x,y)$, $m(x,y)$ provides a  positive integer.
\item $\rm{goal} \in \{\min, \max\}$.


\end{enumerate}
\end{definition}
\noindent
Notice the assumption that the cost is positive, and, in particular, it cannot be zero.


\begin{definition}[Class NPO, \cite{ausiello1999complexity} Def 1.17] An optimization problem $\P = (\I,\S,m, \rm{goal})$ belongs to the class of NP optimization (NPO) problems if the following hold:
\begin{enumerate}
\item The set of instances $\I$ is recognizable in polynomial time.
\item There exists a polynomial $q$ such that given an instance $x\in\I$, for any $y\in \S(x)$, $|y| < q(x)$ and, besides, for any $y$ such that $|y| < q(x)$, it is decidable in polynomial time whether $y\in\S(x)$.
\item The measure function $m$ is computable in polynomial time.
\end{enumerate}
\end{definition}

\begin{definition}[Class PO, \cite{ausiello1999complexity} Def 1.18] An optimization problem $\P$ belongs to the class of PO if it is in NPO and there exists a polynomial-time algorithm that, for any instance $x\in\I$, returns an optimal solution $y\in\S^*(x)$, together with its value $m^*(x)$.
\end{definition}

For intractable problems, it may be acceptable to seek an approximate solution that is sufficiently close to optimal.

\begin{definition}[Approximation Algorithm, \cite{ausiello1999complexity} Def. 3.1]Given an optimization problem $\P = (\I,\S,m,\rm{goal})$ an algorithm $\A$ is an {\em approximation algorithm} for $\P$ if, for any given instance $x\in\I$, it returns an {\em approximate solution}, that is a feasible solution $\A(x) \in \S(x)$.
\end{definition}

\begin{definition}[Performance Ratio, \cite{ausiello1999complexity}, Def. 3.6] \label{def:ratio} Given an optimization problem $\P$, for any instance $x$ of $\P$ and for any feasible solution $y\in\S(x)$, the {\em performance ratio}, {\em approximation ratio} or {\em approximation factor} of $y$ with respect to $x$ is defined as
\begin{align}
R(x,y) = \max\Big\{\frac{m(x,y)}{m^*(x)}, \frac{m^*(x)}{m(x,y)}\Big\},
\end{align}
where $m^*(x)$ is the measure of the optimal solution for the instance $x$.
\end{definition}
Since $m^*(x)$ is a positive integer, the performance ratio is well-defined. It is a rational number in $[1,\infty)$.
Notice that from this definition, it follows that if finding a feasible solution, \eg $y\in\S(x)$, is an NP-hard decision problem, then there exists no polynomial-time approximation algorithm for $\P$, irrespective of the kind of performance evaluation that one could possibly mean. 

\begin{definition}[$r(n)$-approximation, \cite{ausiello1999complexity}, Def. 8.1]Given an optimization problem $\P$ in NPO, an approximation algorithm $\A$ for $\P$, and a function $r\colon \Natural \to (1,\infty)$, we say that $\A$ is an {\em $r(n)$-approximate} algorithm for $\P$ if, for any instance $x$ of $\P$ such that $\S(x) \neq \emptyset$, the performance ratio of the feasible solution $\A(x)$ with respect to $x$ verifies the following inequality:
\begin{align}
R(x,\A(x)) \leq r(|x|).
\end{align}
\end{definition}

\begin{definition}[$F$-APX, \cite{ausiello1999complexity}, Def. 8.2]\label{def:F-APX} Given a class of functions $F$, $F$-APX is the class of all NPO problems $\P$ such that, for some function $r\in F$, there exists a polynomial-time $r(n)$-approximate algorithm for $\P$.
\end{definition}

The class of constant functions for $F$ yields the complexity class APX. Together with logarithmic, polynomial, and exponential functions applied in~\cref{def:F-APX}, the following {\em complexity axis} is established:
\begin{align}\notag                                           
\mbox{PO $\subseteq$ APX $\subseteq$ log-APX $\subseteq$ poly-APX $\subseteq$ exp-APX $\subseteq$ NPO}.
\end{align}

\noindent

Since the measure $m$ needs to be computable in polynomial time for NPO problems, the largest measure and thus the largest performance ratio is an exponential function. But exp-APX is not equal to NPO (assuming P $\neq$ NP) because NPO contains problems whose feasible solutions cannot be found in polynomial time.  For an energy minimization problem, any label assignment is a feasible solution, implying that all energy minimization problems are in exp-APX.


The standard approach for proofs in complexity theory is to perform a reduction from a known NP-complete problem.  Unfortunately, the most common polynomial-time reductions ignore the quality of the solution in the approximated case. For example, it is shown that any energy minimization problem can be reduced to a factor 2 approximable Potts model \cite{prusa2015hard}, however the reduction is not approximation preserving and is unable to show the hardness of general energy minimization in terms of approximation. Therefore, it is necessary to use an approximation preserving (AP) reduction to classify NPO problems that are not in PO, for which only the approximation algorithms are tractable.  AP-preserving reductions preserve the approximation ratio in a linear fashion, and thus preserve the membership in these complexity classes.  Formally,

\begin{definition}[AP-reduction, \cite{ausiello1999complexity} Def. 8.3]\label{def:AP-red}
Let $\P_1$ and $\P_2$ be two problems in NPO. $\P_1$ is said to be AP-{\em reducible} to $\P_2$, in symbols $\P_1 \leqAP \P_2$, if two functions $\pi$ and $\sigma$ and a positive constant $\alpha$ exist such that \footnote{The complete definition contains a rational $r$ for the the two mappings ($\pi$ and $\sigma$) and it is omitted here for simplicity.}:
\begin{enumerate}
\item For any instance $x\in \I_1$, $\pi(x) \in \I_2$.
\item For any instance $x\in \I_1$, if $S_1(x) \neq \emptyset$ then $S_2(\pi(x)) \neq \emptyset$.
\item For any instance $x\in \I_1$ and for any $y \in S_2(\pi(x))$, $\sigma(x, y) \in S_1(x)$.
\item $\pi$ and $\sigma$ are computable by algorithms whose running time is polynomial.
\item For any instance $x\in \I_1$, for any rational $r > 1$, and for any $y \in S_2(\pi(x))$,
\begin{align} \label{eq:AP-red}
R_2(\pi(x),y) \leq r \quad \text{implies} \\
R_1(x, \sigma(x, y)) \leq 1 + \alpha(r-1).
\end{align}
\end{enumerate}
\end{definition}

AP-reduction is the formal definition of the term `as hard as' used in this paper unless otherwise specified. It defines a partial order among optimization problems. With respect to this relationship, we can formally define the subclass containing the hardest problems in a complexity class:

\begin{definition}[$\C$-hard and $\C$-complete,  \cite{ausiello1999complexity} Def. 8.5]\label{def:complete} Given a class $\C$ of NPO problems, a problem $\P$ is $\C$-hard if, for any $\P' \in \C$, $\P' \leqAP \P$. A $\C$-hard problem is $\C$-complete if it belongs to $\C$.
\end{definition}

Intuitively, a complexity class $\C$ specifies the upper bound on the hardness of the problems within, $\C$-hard specifies the lower bound, and $\C$-complete exactly specifies the hardness.
\section{Inapproximability for the General Case} \label{sec:gencase}

In this section, we show that QPBO and general energy minimization are inapproximable by proving they are exp-APX-complete. As previously mentioned, it is already known that these problems are NP-hard \cite{BorosHammer02}, but it was previously unknown whether useful approximation guarantees were possible in the general case.  
The formal statement of QPBO as an optimization problem is as follows:
\begin{problem}{\bf QPBO}
\begin{itemize}
\item[\sc instance:] A pseudo-Boolean function $f \colon \Bool^\V \to \Natural \colon$
\begin{align}
f(x) = \sum_{v\in\V} f_u(x_u) + \sum_{u,v\in\V} f_{uv}(x_u,x_v),
\end{align}
given by the collection of unary terms $f_u$ and pairwise terms $f_{uv}$.
\item[\sc solution:] Assignment of variables $x \in \Bool^\V$.
\item[\sc measure:] $\min f(x) > 0$.
\end{itemize}
\end{problem}

\begin{restatable}{theorem}{Tmain}\label{th:main}
QPBO is exp-APX-complete.
\end{restatable}
\begin{proofsketch} (Full proof in
\ifdefined\NOAPPENDIX
    Appendix A).
\else
    \cref{sec:formalproof}).
\fi
\begin{enumerate}
\item We observe that W3SAT-triv is known to be exp-APX-complete \cite{ausiello1999complexity}. W3SAT-triv is a 3-SAT problem with weights on the variables and an artificial, trivial solution.
\item Each 3-clause in the conjunctive normal form can be represented as a polynomial consisting of three binary variables. Together with representing the weights with the unary terms, we arrive at a cubic Boolean minimization problem.
\item We use the method of~\cite{ishikawa2011transformation} to transform the cubic Boolean problem into a quadratic one, with polynomially many additional variables, which is an instance of QPBO.
\item Together with an inverse mapping $\sigma$ that we define, the above transformation defines an AP-reduction from W3SAT-triv to QPBO, \ie W3SAT-triv $\leqAP$ QPBO. This proves that QPBO is exp-APX-hard.
\item We observe that all energy minimization problems are in exp-APX and thereby conclude that QPBO is exp-APX-complete.
\end{enumerate}
\end{proofsketch}

This inapproximability result can be generalized to more than two labels.
\begin{restatable}{corollary}{Cgen}\label{C:gen}
$k$-label energy minimization is exp-APX-complete for $k \geq 2$.
\end{restatable}
%
\begin{proofsketch}
(Full proof in
\ifdefined\NOAPPENDIX
    Appendix A).
\else
    \cref{sec:formalproof}).
\fi
This theorem is proved by showing QPBO $ \leqAP$ $k$-label energy minimization for $k \geq 2$.
\end{proofsketch}
We show in 
\ifdefined\NOAPPENDIX
    Corollary B.1
\else
    \cref{C:prob-approx}
\fi
the inapproximability in energy (log probability) transfer to probability in Equation~\cref{p-approx-ratio} as well.

Taken together, this theorem and its corollaries form a very strong inapproximability result for general energy minimization \footnote{These results automatically generalize to higher order cases as they subsume the pairwise cases discussed here.}. They imply not only NP-hardness, but also that there is no algorithm that can approximate general energy minimization with two or more labels with an approximation ratio better than some exponential function in the input size. In other words, any approximation algorithm of the general energy minimization problem can perform arbitrarily badly, and it would be pointless to try to prove a bound on the approximation ratio for existing approximation algorithms for the general case.  While this conclusion is disappointing, these results serve as a clarification of grounds and guidance for model selection and algorithm design. Instead of counting on an oracle that solves the energy minimization problem, researchers should put efforts into selecting the proper formulation, trading off expressiveness for tractability.

\section{Inapproximability for the Planar Case} \label{sec:plcase}

\begin{figure}[b]
\begin{center}
   \includegraphics[width=1\linewidth]{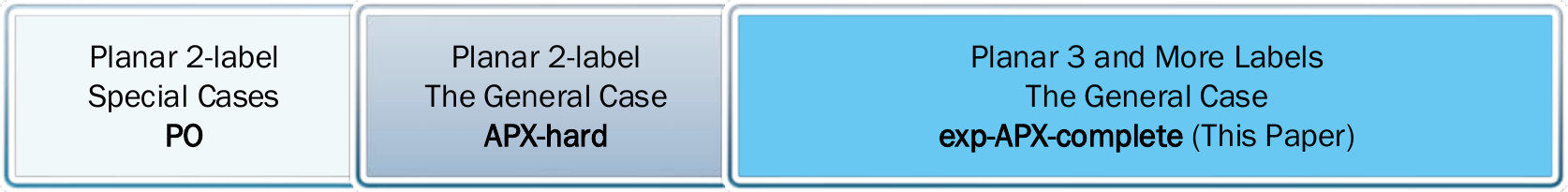}
\end{center}
   \caption{Complexity for planar energy minimization problems. The ``general case'' implies no restrictions on the pairwise interaction type.  This paper shows that the third category of problems is not efficiently approximable.}
\label{fig:planarcomp}
\end{figure}

Efficient algorithms for energy minimization have been found for special cases of 2-label planar graphs.  Examples include planar 2-label problems without unary terms and outerplanar 2-label problems (i.e., the graph structure remains planar after connecting to a common node)~\cite{Schraudolph-10}. 
Grid structures over image pixels naturally give rise to planar graphs in computer vision.  Given their frequency of use in this domain, it is natural to consider the complexity of more general cases involving planar graphs. Figure~\ref{fig:planarcomp} visualizes the current state of knowledge  of the complexity of energy minimization problems on planar graphs.  In this section, we prove that for the case of planar graphs with three or more labels, energy minimization is exp-APX-complete. This result is important because it significantly reduces the space of potentially efficient algorithms on planar graphs. The existence of constant ratio approximation for planar 2-label problems in general remains an open question \footnote{The planar 2-label problem in general is APX-hard, since it subsumes the APX problem planar vertex cover \cite{Bar-Yehuda:1982:AVC:800070.802205}.}.

\begin{restatable}{theorem}{Tplanar}\label{th:planar}
Planar 3-label energy minimization is exp-APX-complete.
\end{restatable}
\begin{proofsketch}
(Full proof in
\ifdefined\NOAPPENDIX
    Appendix A).
\else
    \cref{sec:formalproof}).
\fi
\begin{enumerate}
    \item We construct elementary gadgets to reduce any 3-label energy minimization problem to a planar one with polynomially many auxiliary nodes.
    \item Together with an inverse mapping $\sigma$ that we define, the above construction defines an AP-reduction, \ie, 3-label energy minimization $\leqAP$ planar 3-label energy minimization.
    \item Since 3-label energy minimization is exp-APX-complete (\cref{C:gen}) and all energy minimization problems are in exp-APX, we thereby conclude that planar 3-label energy minimization is exp-APX-complete.
\end{enumerate}
\end{proofsketch}
\begin{restatable}{corollary}{Cplanark}\label{C:planark}
Planar $k$-label energy minimization is exp-APX-complete, for $k \geq 3$.
\end{restatable}
\begin{proofsketch}
(Full proof in
\ifdefined\NOAPPENDIX
    Appendix A).
\else
    \cref{sec:formalproof}).
\fi
This theorem is proved by showing planar 3-label energy minimization $ \leqAP$ planar $k$-label energy minimization, for $k \geq 3$. 
\end{proofsketch}
These theorems show that the restricted case of planar graphs with 3 or more labels is as hard as general case for energy minimization problems with the same inapproximable implications discussed in Section~\ref{sec:gencase}.

The most novel and useful aspect of the proof of Theorem~\ref{th:planar} is the planar reduction in Step 1. The reduction creates an equivalent planar representation to any non-planar 3-label graph.  That is, the graphs share the same optimal value.  The reduction applies elementary constructions or ``gadgets'' to uncross two intersecting edges. This process is repeated until all intersecting edges are uncrossed. Similar elementary constructions were used to study the complexity of the linear programming formulation of energy minimization problems~\cite{prusa2015universality,prusa2015hard}.
Our novel gadgets have three key properties {\em at the same time}: 1) they are able to uncross intersecting edges, 2) they work on non-relaxed problems, \ie, all indicator variables (or pseudomarginals to be formal) are integral; and 3) they can be applied repeatedly to build an AP-reduction.


%

%
The two gadgets used in our reduction are illustrated in Figure~\ref{fig:split}. A 3-label node can be encoded as a collection of 3 indicator variables with a one-hot constraint. In the figure, a solid colored circle denotes a 3-label node, and a solid colored rectangle denotes the equivalent node expressed with indicator variables (white circles).  For example, in Figure~\ref{fig:split}, $a=1$ corresponds to the blue node taking the first label value. The pairwise potentials (edges on the left part of the figures) can be viewed as edge costs between the indicator variables (black lines on the right), \eg, $f_{uv}(3, 2)$ is placed onto the edge between indicator $c$ and $e$ and is counted into the overall measure if and only if $c = e = 1$. In our gadgets, drawn edges represent zero cost while omitted edges represent positive infinity\footnote{A very large number will also serve the same purpose, \eg, take the sum of the absolute value of all energy terms and add 1. Therefore, we are not expanding the set of allowed energy terms to include $\infty$.}.  While the set of feasible solutions remains the same, the gadget encourages certain labeling relationships, which, if not satisfied, cause the overall measure to be infinity. Therefore, the encouraged relationships must be satisfied by any optimal solution. The two gadgets serve different purposes:

\textsc{Split} A 3-label node (blue) is split into two 2-label nodes (green). The shaded circle represents a label with a positive infinite unary cost and thus creates a simulated 2-label node. The encouraged relationships are
\begin{itemize}
    \item $a = 1 \Leftrightarrow d = 1 \text{ and } f = 1$.
    \item $b = 1 \Leftrightarrow g = 1$.
    \item $c = 1 \Leftrightarrow e = 1 \text{ and } f = 1$.
\end{itemize}
Thus $(d,f)$ encodes $a$, $(d,g)$ and $(e,g)$ both encode $b$ and $(e,f)$ encodes $c$.

\textsc{UncrossCopy} The values of two 2-label nodes are encouraged to be the same as their diagonal counterparts respectively (red to red, green to green) without crossing with each other. The orange nodes are intermediate nodes that pass on the values. All types of lines represent the same edge cost, which is 0. The color differences visualize the verification for each of the 4 possible states of two 2-label nodes.  For example, the cyan lines verify the case where the top-left (green) node takes the values (1, 0) and the top-right (red) node takes the value (0, 1).
It is clear that the encouraged solution is for the bottom-left (red) node and the bottom-right (green) node to take the value (0, 1) and (1, 0) respectively.

\begin{figure}[t]
\centering
\begin{tabular}{cc}
\begin{tabular}{c}%
\includegraphics[width=0.39\linewidth]{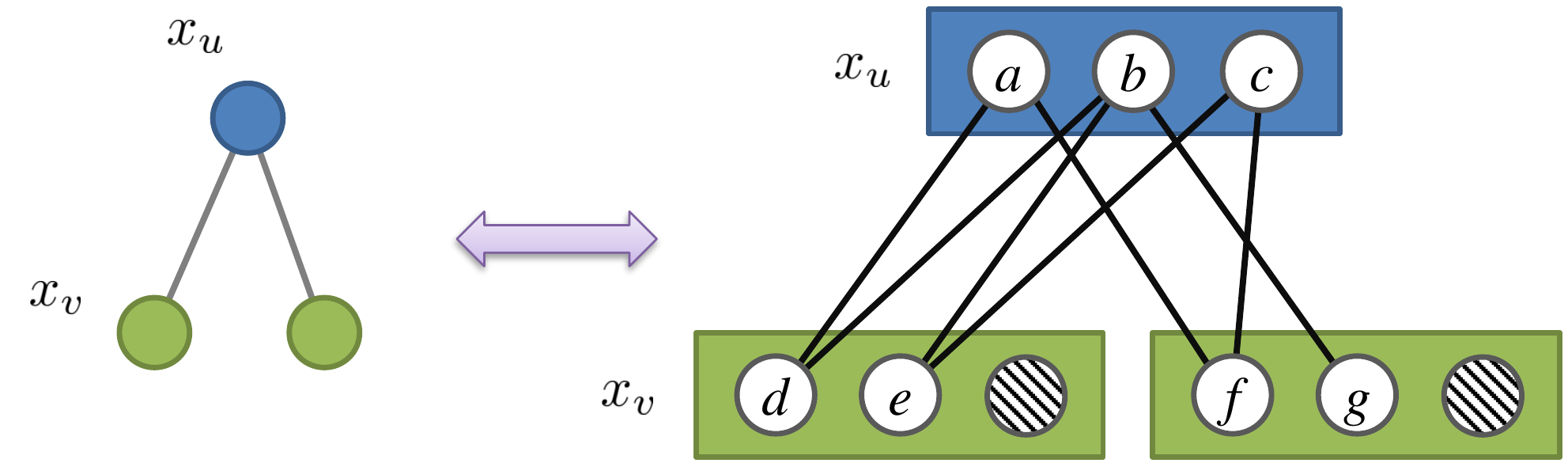}
\end{tabular}&\ \ \ \ \ \ 
\begin{tabular}{c}%
\includegraphics[width=0.39\linewidth]{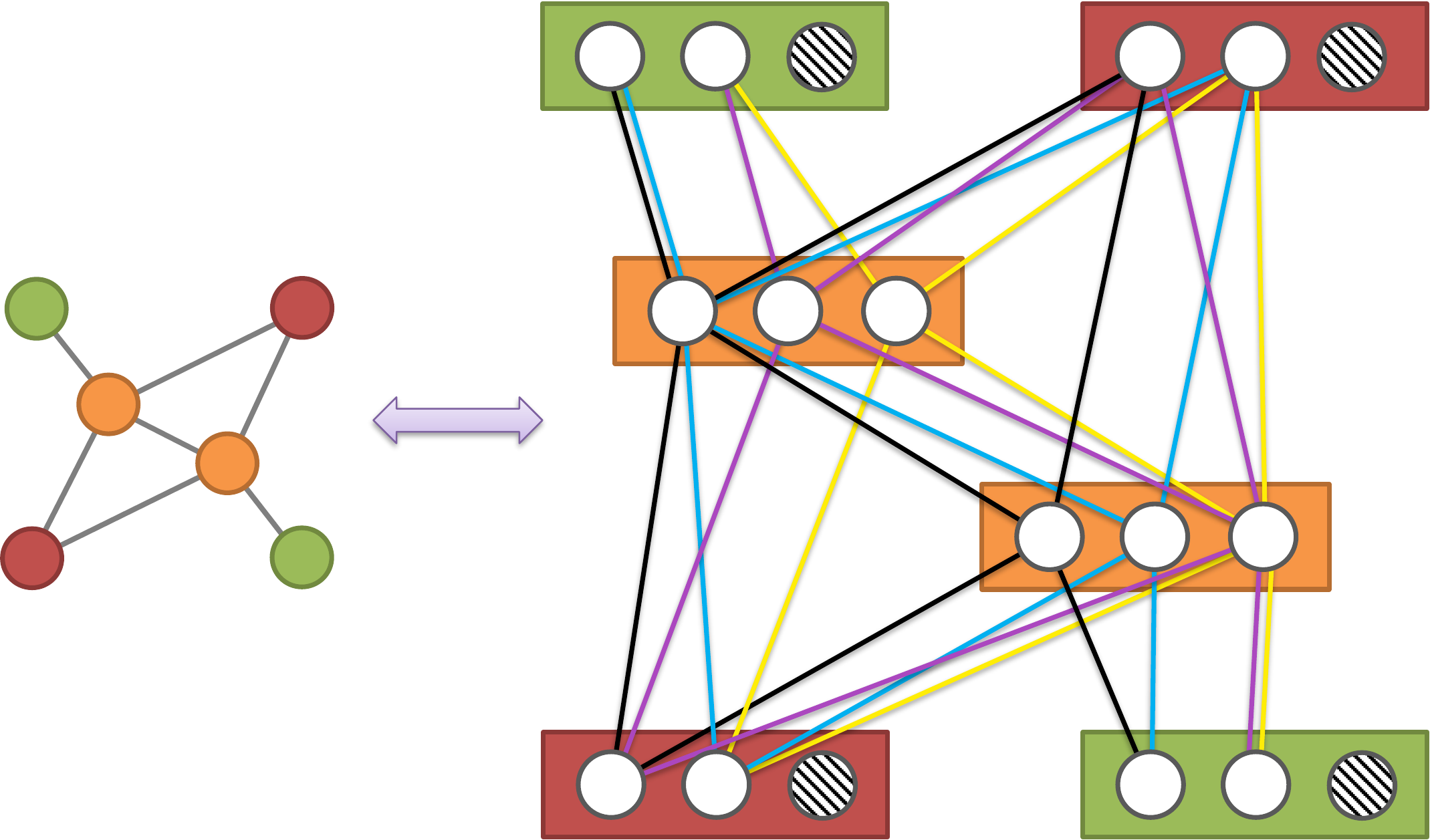}
\end{tabular}\\
{\sc Split} & {\sc UncrossCopy}
\end{tabular}
\caption{Gadgets to represent a 3-label variable as two 2-label variables ({\sc Split}) and to copy the values of two diagonal pairs of 2-label variables without edge crossing ({\sc UncrossCopy}).\label{fig:split}}
\end{figure}

\begin{figure}[b]
\begin{center}
   \includegraphics[trim={0 10.5cm 0 0}, clip, width=0.7\linewidth]{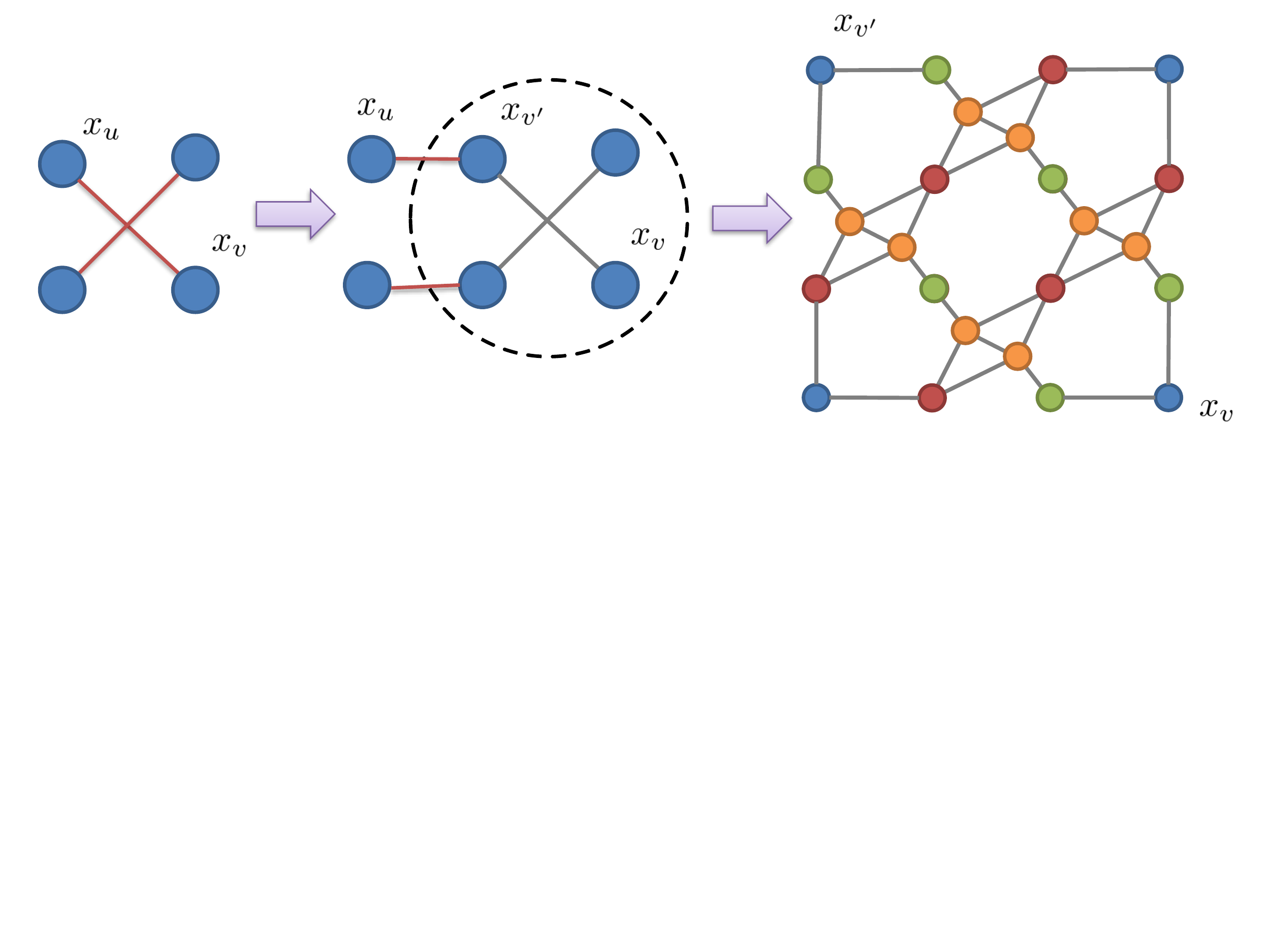}
\end{center}
\caption{Planar reduction for 3-label problems} \label{fig:planarreduct}
\end{figure}

These two gadgets can be used to uncross the intersecting edges of two pairs of 3-label nodes (Figure~\ref{fig:planarreduct}, left).  For a crossing edge ($x_u$, $x_v$), first a new 3-label node $x_{v'}$ is introduced preserving the same arbitrary interaction (red line) as before (Figure~\ref{fig:planarreduct}, middle). Then, the crossing edges (enclosed in the dotted circle) are uncrossed by applying {\sc Split} and {\sc UncrossCopy} four times (Figure~\ref{fig:planarreduct}, right).
Without loss of generality, we can assume that no more than two edges intersect at a common point except at their endpoints.  This process can be applied repeatedly at each edge crossing until there are no edge crossings left in the graph~\cite{prusa2015universality}.

\section{Complexity of Subclass Problems} \label{sec:specialcases}


In this section, we classify some of the special cases of energy minimization according to our complexity axis (Figure~\ref{fig:hardnessaxis}). This classification can be viewed as a reinterpretation of existing results from the literature into a unified framework.

\subsection{Class PO (Global Optimum)}

Polynomial time solvability may be achieved by considering two principal restrictions: those restricting the {\em structure} of the problem, \ie, the graph $G$, and those restricting the type of allowed interactions, \ie, functions $f_{uv}$.

\textbf{Structure Restrictions.} When $G$ is a chain, energy minimization
reduces to finding a shortest path in the trellis graph, which can be solved using a classical dynamic programming (DP) method known as the Viterbi algorithm~\cite{forney1973viterbi}.
The same DP principle applies to graphs of bounded treewidth.  Fixing all variables in a separator set decouples the problem into independent optimization problems. For treewidth 1, the separators are just individual vertices, and the problem is solved by a variant of DP~\cite{Pearl-88,SchlesingerHlavac2002}.
For larger treewidths, the respective optimization procedure is known as junction tree decomposition~\cite{Lauritzen96}. A loop is a simple example of a treewidth 2 problem. However, for a treewidth $k$ problem, the time complexity is exponential in $k$~\cite{Lauritzen96}.
When $G$ is an outer-planar graph, the problem can be solved by the method of~\cite{Schraudolph-10}, which reduces it to a planar Ising model, for which efficient algorithms exist~\cite{Shih-90}.

\textbf{Interaction Restrictions.}
Submodularity is a restriction closely related to problems solvable by minimum cut. A quadratic pseudo-Boolean function $f$ is {\em submodular} iff its quadratic terms are non-positive.  It is then known to be equivalent with finding a minimum cut in a corresponding network~\cite{Hammer:OR68}. 
Another way to state this condition for QPBO is 
$\forall (u, v) \in \E, f_{uv}(0, 1) + f_{uv}(1, 0) \geq f_{uv}(0, 0) + f_{uv}(1, 1).$
%
However, submodularity is more general.  It extends to higher-order and multi-label problems. Submodularity is considered a discrete analog of convexity. Just as convex functions are relatively easy to optimize, general submodular function minimization  can be solved in strongly polynomial time~\cite{Schrijver00}.
Kolmogorov and Zabin introduced submodularity in computer vision and showed that binary 2\textsuperscript{nd} order and 3\textsuperscript{rd} order submodular problems can be always reduced to minimum cut, which is much more efficient than general submodular function minimization~\cite{kolmogorov2004energy}. \citeauthor{Zivny-08-binary} and \citeauthor{ramalingam2008exact} give more results on functions reducible to minimum cut~\cite{Zivny-08-binary,ramalingam2008exact}.
For QPBO on an unrestricted graph structure, the following {\em dichotomy} result has been proven by~\citet{Cohen-04}:
either the problem is submodular and thus in PO or it is NP-hard (\ie, submodular problems are the only ones that are tractable in this case).
%
%

For multi-label problems~\citeauthor{Ishikawa03} proposed a reduction to minimum cut for problems with convex interactions, \ie, where $f_{uv}(x_u,x_v) = g_{uv}(x_u - x_v)$ and $g_{uv}$ is convex and symmetric~\cite{Ishikawa03}. 
It is worth noting that when the unary terms are convex as well, the problem can be solved even more efficiently~\cite{Hochbaum-2001-MRF,Kolmogorov05primal-dualalgorithm}. 
The same reduction~\cite{Ishikawa03} remains correct for a more general class of submodular multi-label problems.
In modern terminology, component-wise minimum $x \wedge y$ and component-wise maximum $x \vee y$ of complete labelings $x$, $y$ for all nodes are introduced ($x,y \in \L^\V$). These operations depend on the {\em order of labels} and, in turn, define a lattice on the set of labelings. The function $f$ is called {\em submodular on the lattice} if $f(x \vee y) + f(x \wedge y) \leq f(x) + f(y)$ for all $x$, $y$~\cite{Topkis-78}.
In the pairwise case, the condition can be simplified to the form of submodularity common in computer vision \cite{ramalingam2008exact}:
$f_{uv}(i, j+1) + f_{uv}(i+1, j) \geq f_{uv}(i, j) + f_{uv}(i+1, j+1).$
In particular, it is easy to see that a convex $f_{uv}$ satisfies it~\cite{Ishikawa03}.
\citet{Kolmogorov-10} and \citet{Arora-12} proposed maxflow-like algorithms for higher order submodular energy minimization. 
\citeauthor{DSchlesinger-07-permuted} proposed an algorithm to find a reordering in which the problem is submodular if one exists~\cite{DSchlesinger-07-permuted}. 
However, unlike in the binary case, solvable multi-label problems are more diverse. A variety of problems are generalizations of submodularity and are in PO, including symmetric tournament pair, submodularity on arbitrary trees, submodularity on arbitrary lattices, skew bisubmodularity, and bisubmodularity on arbitrary domains (see references in~\cite{Thapper-13}).
\citet{Thapper-12} and \citet{Kolmogorov-12-LP-power} characterized these tractable classes and proved a similar dichotomy result: a problem of unrestricted structure is either solvable by LP-relaxation (and thus in PO) or it is NP-hard. It appears that LP relaxation is the most powerful and general solving technique~\cite{Zivny-Werner-Prusa-ASP-MIT2014}.

\textbf{Mixed Restrictions.}
In comparison, results with mixed structure and interaction restrictions are rare. One example is a planar Ising model without unary terms~\cite{Shih-90}. Since there is a restriction on structure (planarity) and unary terms, it does not fall into any of the classes described above. 
Another example is the restriction to supermodular functions on a bipartite graph, solvable by~\cite{DSchlesinger-07-permuted} or by LP relaxation, but not falling under the characterization~\cite{Thapper-13} because of the graph restriction. 

\textbf{Algorithmic Applications.}
The aforementioned tractable formulations in PO can be used to solve or approximate harder problems. Trees, cycles and planar problems are used in dual decomposition methods~\cite{Komodakis-subgradient,KomodakisP08,BatraGPC10}. Binary submodular problems are used for finding an optimized crossover of two-candidate multi-label solutions. An example of this technique, the expansion move algorithm, achieves a constant approximation ratio for the Potts model~\cite{boykov2001approximate}. Extended dynamic programming can be used to solve restricted segmentation problems~\cite{Felzenszwalb-10-Tiered} and as move-making subroutine~\cite{VineetWT12}.
LP relaxation also provides approximation guarantees for many problems~\cite{BachHG15, Chekuri-01, Kleinberg99,Komodakis-07}, placing them in the APX or poly-APX class.

\subsection{Class APX and Class log-APX (Bounded Approximation)} \label{sec:apx}

Problems that have bounded approximation in polynomial time usually have certain restriction on the interaction type. The Potts model may be the simplest and most common way to enforce the smoothness of the labeling. Each pairwise interaction depends on whether the neighboring labellings are the same, \ie $f_{uv}(x_u,x_v) = c_{uv}\delta(x_u, x_v)$.  Boykov \etal showed a reduction to this problem from the NP-hard multiway cut~\cite{boykov2001approximate}, also known to be APX-complete~\cite{ausiello1999complexity, dahlhaus1994complexity}. They also proved that their constructed alpha-expansion algorithm is a 2-approximate algorithm. These results prove that the Potts model is in APX but not in PO. However, their reduction from multiway cut is not an AP-reduction, as it violates the third condition of AP-reducibility. Therefore, it is still an open problem whether the Potts model is APX-complete.
Boykov \etal also showed that their algorithm can approximate the more general problem of metric labeling~\cite{boykov2001approximate}. The energy is called {\em metric} if, for an arbitrary, finite label space $\mathcal{L}$, the pairwise interaction satisfies a) $ f_{uv}(\alpha, \beta) = 0$, b) $f_{uv}(\alpha, \beta) = f_{uv}(\beta, \alpha) \geq 0$, and c) $f_{uv}(\alpha, \beta) \leq f_{uv}(\beta, \gamma) + f_{uv}(\beta, \gamma)$,
for any labels $\alpha$, $\beta$, $\gamma \in \mathcal{L}$ and any $uv \in \E$. Although their approximation algorithm has a bound on the performance ratio, the bound depends on the ratio of some pairwise terms, a number that can grow exponentially large. For metric labeling with $k$ labels, Kleinberg \etal proposed an $O(\log k \log \log k)$-approximation algorithm. This bound was further improved to $O(\log k)$ by \citet{Chekuri:2005:LPF}, making metric labeling a problem in log-APX \footnote{An $O(\log k)$-approximation implies an $O(\log |x|)$-approximation
(see 
\ifdefined\NOAPPENDIX
    Corollary C.1).
\else
    \cref{C:label-approx}).
\fi
}.

We have seen that a problem with convex pairwise interactions is in PO. An interesting variant is its truncated counterpart, \ie, $f_{uv}(x_u,x_v) = w_{uv}\min\{d(x_u - x_v), M\}$, where $w_{uv}$ is a non-negative weight, $d$ is a convex symmetric function to define the distance between two labels, and $M$ is the truncating constant~\cite{veksler2007graph}.
This problem is NP-hard~\cite{veksler2007graph}, but \citet{Kumar-11-improved} have proposed an algorithm that yields bounded approximations with a factor of $2+\sqrt{2}$ for linear distance functions and a factor of $O(\sqrt{M})$ for quadratic distance functions\footnote{In these truncated convex problems, the ratio bound is defined for the pairwise part of the energy \cref{eq:1}. The approximation ratio in accordance to our definition is obtained assuming the unary terms are non-negative.}. This bound is analyzed for more general distance functions by~\citet{kumar2014rounding}.

Another APX problem with implicit restrictions on the interaction type is logic MRF~\cite{bach2015unifying}. It is a powerful higher order model able to encode arbitrary logical relations of Boolean variables. It has energy function $f(x) = \sum_i^n w_iC_i$, where each $C_i$ is a disjunctive clause involving a subset of Boolean variables $x$, and $C_i = 1$ if it is satisfied and 0 otherwise. Each clause $C_i$ is assigned a non-negative weight $w_i$. The goal is to find an assignment of $x$ to maximize $f(x)$. As disjunctive clauses can be converted into polynomials, this is essentially a pseudo-Boolean optimization problem. However, this is a special case of general 2-label energy minimization, as its polynomial basis spans a subspace of the basis of the latter. \citet{bach2015unifying} proved that logic MRF is in APX by showing that it is a special case of MAX-SAT with non-negative weights.
%

\section{Discussion}\label{sec:discussion}

The algorithmic implications of our inapproximability have been discussed above. Here, we focus on the discussion of practical implications. 
The existence of an approximation guarantee indicates a practically relevant class of problems where one may expect reasonable performance. In structural learning for example, it is acceptable to have a constant factor approximation for the inference subroutine when efficient exact algorithms are not available. \citeauthor{finley2008training} proved that this constant factor approximation guarantee yields a multiplicative bound on the learning objective, providing a relative guarantee for the quality of the learned parameters~\cite{finley2008training}. An optimality guarantee is important, because the inference subroutine is repeatedly called, and even a single poor approximation, which returns a not-so-bad worst violator, will lead to the early termination of the structural learning algorithm. 

However, despite having no approximation ratio guarantee, algorithms such as the extended roof duality algorithm for QPBO~\cite{rother2007optimizing} are still widely used. This gap between theory and application applies not only to our results but to all other complexity results as well. We list several key reasons for the potential lack of correspondence between theoretical complexity guarantees and practical performance.

\textbf{Complexity results address the worst case scenario.} Our inapproximability result guarantees that for any polynomial time algorithm, there exists an input instance for which the algorithm will produce a very poor approximation. However, applications often do not encounter the worst case. Such is the case with the simplex algorithm, whose worst case complexity is exponential, yet it is widely used in practice. 




\textbf{Objective function is not the final evaluation criterion.} In many image processing tasks, the final evaluation criterion is the number of pixels correctly labeled. The relation between the energy value and the accuracy is implicit. In many cases, a local optimum is good enough to produce a high labeling accuracy and a visually appealing labeling.

\textbf{Other forms of optimality guarantee or indicator exist.} 
Approximation measures in the distance of solutions or in the expectation of the objective value are likely to be prohibitive for energy minimization, as they are for Bayesian networks~\cite{Kwisthout-11,Kwisthout-13, kwisthout2015tree}.
On the other hand, a family of energy minimization algorithms has the property of being {\em persistent} or {\em partial optimal}, meaning a subset of nodes have consistent labeling with the global optimal one~\cite{Boros:TR91-maxflow, BorosHammer01}. Rather than being an optimality guarantee, persistency is an optimality indicator. In the worst case, the set of persistent labelings could be empty, yet the percentage of persistent labelings over the all the nodes gives us a notion of the algorithm's performance on this particular input instance. Persistency is also useful in reducing the size of the search space \cite{kohli2008partial, SSS-15-IRI}. Similarly, the per-instance integrality gap of duality based methods is another form of optimality indicator and can be exponentially large for problems in general \cite{Komodakis-07, Sontag-12}.

\section{Conclusion}

In this paper, we have shown inapproximablity results for energy minimization in the general case and planar 3-label case. In addition, we present a unified overview of the complexity of existing energy minimization problems by arranging them in a fine-grained complexity scale. These altogether set up a new viewpoint for interpreting and classifying the complexity of optimization problems for the computer vision community.  In the future, it will be interesting to consider the open questions of the complexity of structure-, rank-, and expectation-approximation for energy minimization.

\section*{Acknowledgements}
This material is based upon work supported by the National Science Foundation under Grant No. IIS-1328930 and by the European Research Council under the Horizon 2020 program, ERC starting grant agreement 640156.

{\small
\renewcommand{\bibname}{\protect\leftline{References\vspace{-2ex}}}
\bibliographystyle{splncsnat}
\bibliography{bib/strings,bib/egbib,bib/all}
}

\ifdefined\NOAPPENDIX
\else
    \newpage
    \appendix
    \setboolean{InAppendix}{true}

    \title{Complexity of Discrete Energy Minimization Problems (ECCV'16 Appendix)}

\author{Mengtian Li \qquad Alexander Shekhovtsov \qquad Daniel Huber}
\institute{}

\authorrunning{Mengtian Li, Alexander Shekhovtsov and Daniel Huber}

\maketitle

\section{Formal Proofs}\label{sec:formalproof}


Note for all proofs in this section, we assign integer values to Boolean functions: 0 for False and 1 for True.


\subsection{General Case}


\Tmain*
\begin{proof}
We reduce from the following problem.
\begin{problem}[\cite{ausiello1999complexity}, Section 8.3.2]{\bf W3SAT-triv}
\begin{itemize}
\item[\sc instance:] Boolean CNF formula $F$ with variables $x_1,\cdots,x_n$ and each clause assuming exactly 3 variables; non-negative integer weights $w_1,\cdots, w_n$.
\item[\sc solution:] Truth assignment $\tau$ to the variables that either satisfies $F$ or assigns the trivial, all-true assignment.
\item[\sc measure:] $\min \sum_{i=1}^n w_i \tau(x_i)$.
\end{itemize}
\end{problem}

W3SAT-triv is known to be exp-APX-complete \cite{ausiello1999complexity}. We use an AP-reduction from W3SAT-triv to prove the same completeness result for QPBO. The optimal value of W3SAT-triv is upper bounded by $M := \sum_{i}w_i$ because the all-true assignment is feasible.
The objective weight is represented in QPBO as unary terms $f_i(x_i) = w_i x_i$. 
For every Boolean clause $C(x_i,x_j,x_k) \in F$ we construct a triple-wise term
\begin{align}
\phi_{ijk}(x_i,x_j,x_k) = M(1-C(x_i,x_j,x_k)).
\end{align}
This term takes the large value $M$ iff $C$ is not satisfied and $0$ otherwise. Further, the Boolean clause $C(x_i,x_j,x_k)$ can be represented uniquely as a multi-linear cubic polynomial. For example, a clause $x_1 \lor {\bar x}_2 \lor {\bar x}_3$ can be represented as
\begin{align}
    1 - (1 - x_1) x_2 x_3 = x_1 x_2 x_3 - x_2 x_3 + 1.
\end{align}
Then we obtain similar representation with a single third order term and a second order multi-linear polynomial for $\phi_{ijk}$:
\begin{align}
\phi_{ijk} = M(a x_i x_j x_k + \sum_{J} b_{J} \prod_{l\in J} x_l),
\end{align}
where $J\subseteq \{i,j,k\}, |J|\leq 2$, $\prod_{l\in J} x_l$ is set to 1 if $J$ is empty, $a \in \{-1, 1\}$, and $b_{J} \in \{-1, 0, 1\}$. We now apply the quadratization techniques \cite{ishikawa2011transformation} to $\phi_{ijk}$. After introducing an auxiliary variable $x_w$ with $w > n$, we observe the following identities:
\begin{align}
-x_i x_j x_k & = \min_{x_w\in \{0, 1\}} -x_w (x_i+x_j+x_k-2) \\
x_i x_j x_k & =  \min_{x_w\in \{0, 1\}} \big( (x_w{-}1) (x_i{+}x_j{+}x_k{-}1) + (x_i x_j{+}x_i x_k{+}x_j x_k) \Big)
\end{align}
In either case, substituting the cubic term $a x_i x_j x_k$ in $\phi_{ijk}$ with the expression inside the min operator, we can have a unified quadratic form
\begin{align}
\psi_{ijk} := M \sum_{J_w}b_{J_w}\prod_{l\in J_w} x_l,
\end{align}
where $J_w \subseteq \{i, j, k, w\}, |J_w|\leq 2$ and $\prod_{i\in J_w} x_i$ is set to 1 if $J_w$ is empty. In both cases, the quadratic form takes the same optimal values as its cubic counterpart given the optimal assignment, i.e.,
\begin{align}
 \min_{x_i, x_j, x_k, x_w} \psi_{ijk} = \min_{x_i, x_j, x_k} \phi_{ijk},
\end{align}
but the transformation expands the original range of the cubic term from $\{-1, 0\}$ to $\{-1, 0, 1, 2\}$ and from $\{0, 1\}$ to $\{0, 1, 3\}$ respectively for $a = -1$ and $a = 1$. Therefore, the cost of the constructed instance of QPBO is bounded in the absolute value by $3M$ and the number of added variables is exactly the number of clauses in $F$. Clearly, this construction can be computed in polynomial time. {\em Note that when approximation is used, this transformation is no longer exact ($\psi_{ijk} \neq \phi_{ijk}$), as the optimality of the auxiliary variable $x_w$ cannot be guaranteed. However, it can be verified that under all possible assignments (ignoring the min operator) in either case, $\psi_{ijk} \geq 0$, which is the key for the reduction to be an approximation preserving (AP) one.}

The construction above defines a mapping $\pi$ from any instance of W3SAT-triv ($p_1 \in I_1$) to an instance of QPBO ($p_2 \in I_2$). The mapping $\sigma$ from feasible solutions of $p_2$ ($x \in S_2(p_2)$) to that of $p_1$ is defined as follows: if $f(x) \geq M$, then let the mapped solution $\sigma(p_1, x)$ be the all true assignment, otherwise let the mapped solution $\sigma(p_1, x)$ be $x_i, i \in \{1, ..., n\}$. 

Now, we need to show that $(\pi,\sigma)$ together with a constant $\alpha$ is an AP-reduction. 
Let $m_1$, $m_2$, $m_1^*$ and $m_2^*$ to be short for $m_1(p_1, \sigma(p_1, x))$, $m_2(p_2)$, $m_1^*(p_1)$, and $m_2^*(\pi(p_2))$ respectively, where $*$ indicates the optimal solution.  First, note that $\sigma(p_1, x)$ is always feasible for {\em W3SAT}-triv: either it satisfies $F$ or $f(x) \geq M$ and therefore $\sigma(p_1, x)$ is the all-true assignment. In the first case, since every quadratic term is non-negative, we have
\begin{align}
& m_1 = \sum_{i=1}^n x_i w_i \\
& \leq \sum_{i=1}^n x_i w_i + \sum_{C_{ijk} \in F} \psi_{ijk}(x_i,x_j,x_k) = f(x) = m_2.
\end{align}
In the second case, by construction
\begin{align}
m_1= M \leq f(x) = m_2.
\end{align}
Therefore, no matter which case $m_1 \leq m_2$.
\par
Now for the optimal solution, if $F$ is satisfiable, then by construction $m_1^* = m_2^*$. Recall from \cref{def:ratio}, $R = m/m^*$. 
For any instance $p_1 \in I_1$, for any rational $r > 1$, and for any $x \in S_2(p_2)$, if
\begin{align}
R_2(p_2, x) \leq r,
\end{align}
then
\begin{align}
m_1 \leq m_2 \leq rm_2^* = rm_1^* \\
R_1(p_1, \sigma(p_1, x)) = \frac{m_1}{m_1^*} \leq r
\end{align}
If $F$ is not satisfiable, $m_1^* = M \leq m_2^*$ and $m_2 \geq m_2* \geq M$. Thus, for any instance $p_1 \in I_1$, for any rational $r > 1$, and for any $x \in S_2(p_2)$,
\begin{align}
R_1(p_1, \sigma(p_1, x)) = \frac{m_1}{m_1^*} = \frac{M}{M} = 1 \leq r
\end{align}
Therefore $(\pi,\sigma, 1)$ is an AP-reduction. Since W3SAT-triv is exp-APX-complete and QPBO is in exp-APX, we prove that QPBO is exp-APX-complete.
\end{proof}


\Cgen*
\begin{proof}
We create an AP-reduction from QPBO to $k$-label energy minimization by setting up the unary and pairwise terms to discourage a labeling with the additional $k - 2$ labels.

Denote QPBO as $ \P_1 = (\I_1, \S_1, m_1, \min)$ and $k$-label energy minimization as $ \P_2 = (\I_2, \S_2, m_2, \min)$.
Given an instance $p_1 = (\G = (\V, \E), \L_1, f) \in \I_1$, let $M(p_1)$ be a large number such that all for all ${\bf x}_1 \in \S_1$, $m_1 < M$. For example, we can let 
\begin{align}
    M = \sum_{u \in \V}\sum_{x_u \in \L_1}|f_u(x_u)| + \sum_{(u, v) \in \E}\sum_{x_u \in \L_1}\sum_{x_v \in \L_1}|f_{uv}(x_u, x_v)| + 1. \label{eq:m}
\end{align}

We define the forward mapping $\pi$ from any $p_1 \in I_1$ to $p_2 = (\G = (\V, \E), \L_2, g) \in I_2$ as follows:
\begin{itemize}
    \item $g_u(a) = f_u(a)$, for $\forall a \in \L_1$, and $\forall u \in \V$;
    \item $g_u(a) = M$, for $\forall a \notin \L_1$, and $\forall u \in \V$;
    \item $g_{uv}(a, b) = f_{uv}(a,b)$, for $\forall a, b \in \L_1$, and $\forall (u, v) \in \E$;
    \item $g_{uv}(a, b) = M$ if either $a$ or $b$ $\notin \L_1$ for $\forall (u, v) \in \E$.
\end{itemize}
    
This setup has two properties:
\begin{itemize}
    \item $m_2 \geq M$ if and only if the labeling ${\bf x}_2 \in \S_2$ includes labels that are not in $\L_1$;
    \item $m_1^*$ = $m_2^*$, for any $p_1$ and $p_2 = \pi(p_1)$.
\end{itemize}

Then we define the reverse mapping $\sigma$ from any $(p_2, {\bf x}_2)$ to ${\bf x}_1 \in \S_1$ to be
\begin{itemize}
    \item ${\bf x}_1$ = ${\bf x}_2$, if $m_2 < M$;
    \item ${\bf x}_1$ be any fixed feasible solution (e.g., all nodes are labeled as the first label), if $m_2 \geq M$.
\end{itemize}

Observe that in both cases, $m_1 \leq m_2$. For any instance $p_1 \in I_1$, for any rational $r > 1$, and for any ${\bf x}_2 \in S_2$, if
\begin{align}
R_2(p_2, {\bf x}_2) = \frac{m_2}{m_2^*} \leq r,
\end{align}
then
\begin{align}
m_1 \leq m_2 \leq rm_2^* = rm_1^* \\
R_1(p_1, {\bf x}_1) = \frac{m_1}{m_1^*} \leq r
\end{align}
Therefore $(\pi,\sigma, 1)$ is an AP-reduction. As QPBO is exp-APX-complete and all energy minimization problems are in exp-APX, we conclude that $k$-label energy minimization is exp-APX-complete for $k \geq 2$.
\end{proof}

The above construction also formally shows that the energy minimization problem can only become harder when having a larger labeling space, irrespective of the graph structure and the interaction type.

\subsection{Planar Case}\label{sec:prfplanar}


\Tplanar*
\begin{proof}



We create an AP-reduction from 3-label energy minimization to planar 3-label energy minimization by introducing polynomially many auxiliary nodes and edges.

Denote 3-label energy minimization as $ \P_1 = (\I_1, \S_1, m_1, \min)$ and planar 3-label energy minimization as $\P_2 = (\I_2, \S_2, m_2, \min)$.
Given an instance $p_1 \in \I_1$, we compute a large number $M(p_1)$ as in Equation~\cref{eq:m} in the proof for \cref{C:gen}.

The gadget-based reduction presented in Section~\ref{sec:plcase}, defines a forward mapping $\pi$ from any $p_1 = (\G_1 = (\V_1, \E_1), \L, f) \in I_1$ to $p_2 = (\G_2 = (\V_2, \E_2), \L, g) \in I_2$. Let $\V_3$ be the nodes added during the reduction, then $\V_2 = \V_1 \cup \V_3$. The two gadgets {\sc Split} and {\sc UncrossCopy} are used 4 times each to replace an edge crossing (point of intersection not at end points) with a planar representation (Figure~\ref{fig:planarreduct}), introducing 22 auxiliary nodes. Since the gadgets can be drawn arbitrarily small so that they are not intersecting with any other edges, we can repeatedly replace all edge crossings in $\G_1$ with this representation. There can be up to $O(|\E_1|^2)$ edge crossings,
and we have $|\mathcal{V}_3|$ = $O(|\E_1|^2)$.
Given that the reduction adds only a polynomial number of auxiliary nodes, the forward mapping $\pi$ can be computed by a polynomial time algorithm.

This setup has two properties:
\begin{itemize}
    \item $m_2 \leq M$ if and only if the labeling ${\bf x}_1$ is the same as the partial labeling in ${\bf x}_2$ restricting to nodes in $\V_1$ in $\G_2$.
    \item $m_1^*$ = $m_2^*$, for any $p_1$ and $p_2 = \pi(p_1)$.
\end{itemize}

Then we define the reverse mapping $\sigma$ from any $(p_2, {\bf x}_2)$ to ${\bf x}_1 \in \S_1$ to be
\begin{itemize}
    \item ${\bf x}_1$ = ${\bf x}_2$, if $m_2 < M$;
    \item ${\bf x}_1$ be any fixed feasible solution (e.g., all nodes are labeled as the first label), if $m_2 \geq M$.
\end{itemize}

Observe that in both cases, $m_1 \leq m_2$. 
For any instance $p_1 \in I_1$, for any rational $r > 1$, and for any ${\bf x}_2 \in S_2$, if
\begin{align}
R_2(p_2, {\bf x}_2) = \frac{m_2}{m_2^*} \leq r,
\end{align}
then
\begin{align}
m_1 \leq m_2 \leq rm_2^* = rm_1^* \\
R_1(p_1, {\bf x}_1) = \frac{m_1}{m_1^*} \leq r
\end{align}
Therefore $(\pi,\sigma, 1)$ is an AP-reduction. As 3-label energy minimization is exp-APX-complete (\cref{C:gen}) and all energy minimization problems are in exp-APX, we conclude that planar 3-label energy minimization is exp-APX-complete.

\end{proof}

\Cplanark*
\begin{proof}
The proof of \cref{C:gen} is graph structure independent. Therefore, the same proof applies here.
\end{proof}

\section{Relation to Bayesian Networks}\label{sec:BayesNet}
There are substantial differences between results for Bayesian networks~\cite{Abdelbar-98} and our result.
Bayesian networks have a probability density function $p(x)$ that factors according to a directed acyclic graph, \eg, as $p(x_1,x_2,x_3) = p(x_1 | x_2,x_3)p(x_2)p(x_3)$.
Finding the MAP assignment (same as the most probable estimate (MPE)) in a Bayesian network is related to energy minimization~\eqref{eq:1} by letting $f(x) = -\log(p(x))$. The product is transformed into the sum and so, \eg, factor $p(x_1 | x_2,x_3)$ corresponds to term $f_{1,2,3}(x_1,x_2,x_3)$.
\par
The inapproximability result of~\citet{Abdelbar-98} holds even when restricting to binary variables and factors of order three. However, \cite[Section 6.1]{Abdelbar-98} count incoming edges of the network. For a factor $p(x_1 | x_2,x_3)$, there are two, but the total number of variables it couples is three and therefore such a network does not correspond to QPBO. If one restricts to factors of at most two variables, \eg, $p(x_1 | x_2)$, in a Bayesian network, then only tree-structured models can be represented, which are easily solvable. 
\par
In the other direction, representing pairwise energy~\eqref{eq:1} as a Bayesian network may require to use factors of order up to $|\mathcal{V}|$ composed of conditional probabilities of the form $p(x_i \mid x_j, x_k, \cdots )$ with the number of variables depending on the vertex degrees. It is seen that while the problems in their most general forms are convertible, fixed-parameter classes (such as order and graph restrictions) differ significantly. In addition, the approximation ratio for probabilities translates to an absolute approximation (an additive bound) for energies. The next corollary of our main result illustrates this point.

\begin{corollary}\label{C:prob-approx}
It is NP-hard to approximate MAP in the value of probability~\eqref{p-approx-ratio} with any exponential ratio $\exp(r(n))$, where $r$ is polynomial.
\end{corollary}


\begin{proof}
Recall that the probability $p(x)$ is given by the exponential map of the energy: $p(x) = \exp(-f(x))$. Assume for contradiction that there is a polynomial time algorithm $\A$ that finds solution $x$ and a polynomial $r(n) \geq 0$ for $n > 0$ such that
\begin{align}
\frac{p(x^*)}{p(x)} \leq e^{r(n)}
\end{align}
for all instances of the problem.
Taking the logarithm,
\begin{align}
-f(x^*) + f(x) \leq r(n).
\end{align}
or, 
\begin{align}
f(x) \leq r(n) + f(x^*).
\end{align}
Divide by $f(x^*)$, which, by definition of NPO is positive, we obtain
\begin{align}\label{f-approx}
\frac{f(x)}{f(x^*)} \leq 1 + \frac{1}{f(x^*)}r(n) \leq 1 + r(n).
\end{align}
where we have used that $f(x^*)$ is integer and positive and hence it is greater or equal to 1. Inequality~\eqref{f-approx} provides a polynomial ratio approximation for energy minimization. Since the latter is exp-APX-complete (\cref{C:gen}), this contradicts existence of the polynomial algorithm $\A$, unless P = NP.
\end{proof}

Note, this corollary provides a stronger inapproximability result for probabilities than was proven in~\cite{Abdelbar-98}.

\begin{remark}~\citet{Abdelbar-98} have shown also the following interesting facts. For Bayesian networks, the following problems are also APX-hard (in the value of probability):
\begin{itemize}
\item Given the optimal solution, approximate the second best solution;
\item Given the optimal solution, approximate the optimal solution conditioned on changing the assignment of one variable.
\end{itemize}
\end{remark}


\section{Miscellaneous}

This result is used in Section~\ref{sec:apx}.
\begin{corollary} \label{C:label-approx}
An $O(\log k)$-approximation implies an $O(\log |x|)$-approximation for $k$-label energy minimization problems.
\end{corollary}
\begin{proof}
Observe that an instance of the energy minimization problem \cref{eq:1} is completely specified by a set of all unary terms $f_u$ and pairwise terms $f_{uv}$. This defines a natural encoding scheme to describe an instance of an energy minimization problem with binary alphabet $\{0, 1\}$. Assume each potential is encoded by $d$ digits, the input size
\begin{align}
|x| = O((k|\mathcal{V}| + k^2|\mathcal{V}|^2)d) = O(k^2|\mathcal{V}|^2).
\end{align}
For an $O(\log k)$-approximation algorithm, the performance ratio
\begin{align}
    r = O(\log k) = O(\log k + \log |\mathcal{V}|) = O(\log k|\mathcal{V}|) = O(\log|x|),
\end{align}
which implies an $O(\log |x|)$-approximation algorithm.

\end{proof}

\fi



\end{document}